\documentclass{article}

\usepackage{PRIMEarxiv}

\usepackage[utf8]{inputenc} 
\usepackage[T1]{fontenc}    
\usepackage{hyperref}       
\usepackage{url}            
\usepackage{booktabs}       
\usepackage{amsfonts}       
\usepackage{nicefrac}       
\usepackage{microtype}      
\usepackage{lipsum}
\usepackage{fancyhdr}       
\usepackage{graphicx}       
\graphicspath{{media/}}     

\usepackage[utf8]{inputenc}

\usepackage{amsmath,amsfonts,bm}

\def\eqref#1{equation~\ref{#1}}

\def\1{\bm{1}}

\def\vx{{\bm{x}}}
\def\vy{{\bm{y}}}

\DeclareMathAlphabet{\mathsfit}{\encodingdefault}{\sfdefault}{m}{sl}
\SetMathAlphabet{\mathsfit}{bold}{\encodingdefault}{\sfdefault}{bx}{n}

\newcommand{\R}{\mathbb{R}}

\DeclareMathOperator*{\argmin}{arg\,min}

\usepackage{url}

\usepackage{pdflscape}
\usepackage{amsthm}
\usepackage{dsfont}
\usepackage{mathtools} 
\usepackage[ruled,vlined]{algorithm2e}
\usepackage{graphicx}
\usepackage{booktabs}
\usepackage{multirow}
\usepackage{color,soul}

\usepackage{lipsum}
\usepackage{dirtytalk}
\usepackage{lineno}
\usepackage{bbm, dsfont}
\usepackage[numbers]{natbib}
\usepackage{xcolor}
\usepackage{hyperref}
\hypersetup{
    colorlinks=false,
    linkbordercolor=red
}
\usepackage{caption}

\newtheorem{theorem}{Theorem}
\newtheorem{definition}{Definition}
\newtheorem{lemma}{Lemma}

\title{Neural Multi-Quantile Forecasting for Optimal Inventory Management}

\author{
  Federico Garza Ramírez \\
  ITAM \\
  \texttt{fede.garza.ramirez@gmail.com} \\
}

\begin{document}
\maketitle

\begin{abstract}
In this work we propose the use of quantile regression and dilated recurrent neural networks with temporal scaling (MQ-DRNN-s) and apply it to the inventory management task. This model showed a better performance of up to 3.2\% over a statistical benchmark (the quantile autoregressive model with exogenous variables, QAR-X), being better than the MQ-DRNN without temporal scaling by 6\%. The above on a set of 10,000 time series of sales of El Globo over a 53-week horizon using rolling windows of 7-day ahead each week.
\end{abstract}

\keywords{Deep Learning \and Time Series \and Quantile Regression}

\section{Introduction}

With the increasing generation of data over time and ever-increasing storage capacity, companies face the opportunity to make better decisions. Time series forecasting has proven particularly useful to leverage this information in inventory management. Inventory management benefits from improved forecasting and the generation of different scenarios, which translates into less transportation (reducing costs and CO2 emissions) and better production planning (less waste), as noted by \cite{oreshkin2020metalearning}.




In this work, inspired by the recent success of neural networks in the datasets of the third and fourth Makridakis competitions for time series forecasting (M3, \cite{MAKRIDAKIS2000451} and M4, \cite{makridakis2018m4competition_results}) and the Tourism dataset \citep{ATHANASOPOULOS2011822}, and the necessity to improve the quality of forecasts for the inventory management task, we propose a model based on \emph{Dilated Recurrent Neural Networks} capable of generating multiple quantiles with adaptive scaling (MQ-DRNN-s) that proved to be up to 3.1\% better than statistical models. These results were obtained on 10,000 products from one of the largest retail companies in Mexico for 53 weeks.

The work developed by \cite{chang2017dilatedRNN} is the main inspiration for the model described in this document. They propose the \emph{Dilated Recurrent Neural Network} as a faster and more accurate alternative to the \emph{Recurrent Neural Network}, which is consistent with the results found in this paper. The MQ-DRNN-s explored here also includes data preprocessing within the model, a fundamental feature for the proper functioning of machine learning models. This data preprocessing handled by the model facilitates its deployment to production since less data engineering is required to obtain reasonably good results.

In addition to this introduction, this document’s structure includes 5 sections as follows. Section \ref{section:literature} contains a review on literature on optimal inventory management and deep learning, Section \ref{section:method} introduces notation and describes the MQ-DRNN-s model. Section \ref{section:experiments} explores the performance of the model on the task of peak demand prediction. 
Finally Section \ref{section:conclusion} comments conclusions about this work.

The contribution of this work is framed in the literature of time series forecasting using machine learning and contributes to three specific aspects.

\begin{enumerate}
    \item \textbf{Non-parametric quantile estimation:} A model is proposed to non-parametrically forecast multiple quantiles simultaneously from a set of time series. This is achieved through the quantile loss.
    \item \textbf{Temporal scaling technique:} The model itself handles the necessary processing of the data to obtain better results than statistical models. This characteristic facilitates its use in production environments.
    \item \textbf{Forecasting comparative analysis:} Finally, we provide a comparison with statistical models and models similar to ours without the inclusion of our proposed preprocessing. Our proposed model outperformed classical statistical benchmarks by 3.2\%. And the naively implemented MQ-DRNN without scaling by 6\%.
\end{enumerate}

\section{Notation}

In this section, we briefly introduce the notation used throughout the work. We denote vectors in bold and scalars in non-bold. A time-series vector is denoted by $\vy$, while a particular value for the observation at time $t$ is denoted by $y_t$. The subscript $ T$ denotes the last observation in the time series; that is, $y_T$. $H$ gives the prediction horizon; meanwhile, we denote by $\hat{y}_{T+h}$ the forecasts associated with this horizon, where $1 \leq h \leq H$ and by $\hat{y}^{(q)}_{T+h}$ the forecast of the quantile $q$ at $T+h$; the actual values are denoted by $y_{T+h}$ and $y^{(q)}_{T+h}$ respectively. We use the subscript $i$ when it is necessary to denote different time series, i.e., $\vy_i$. Finally, we denote the set of natural numbers less than or equal to $d \in \mathbb{N}$ as $[d]$.

\section{Optimal Inventory Managment}

In time series forecasting, it is standard to estimate the conditional expectation of the time series in the future. Although it is a widespread practice, there are scenarios where it is convenient to generate forecasts above or below the mean value. Quantile forecasting addresses the generation of such scenarios. Theoretically, 
the problem of the optimal inventory allocation has been explored in the \emph{News Vendor Problem} from operations research literature \citep{qin2011newsvendor_review}. Its solution justifies the choice of a quantile; and states that maximizing News Vendor profits is given by choosing a quantile.

As described by \cite{qin2011newsvendor_review}, the newsvendor problem is an excellent framework to work on inventory-related problems. At its most basic setting, the problem involves three economic agents: the supplier, the buyer (the newsvendor), and the consumers, interacting within a single period. 

This problem focuses on determining an optimal stock quantity $y^s$ of a single good (newspapers) for the newsvendor to buy from the supplier at a given cost $v > 0$. From the consumers' side, their demanded quantity is given by the random variable $Y$ (with density and distribution functions $f_Y$ and $F_Y$, respectively), at an exogenous price $p > v$. The supplier side assumes that there are no capacity restrictions and zero transactional costs. During the single period, the quantity $y^s$, a priori chosen by the newsvendor, might exceed or fall behind the realization of $Y$.

Finally, if the stocked quantity $y^s$ exceeds the quantity $y^d$ demanded by the consumers, the remaining units have a unitary recovery cost of $g < v$; if the opposite happens, then each unit of demand not met has a cost $B > 0$ for the buyer.

Thus, the profits function that reflects the problem the buyer faces is 

\begin{align}
\begin{split}
\Pi(y^s,y^d) &= \left(py^d - vy^s + g(y^s - y^d)\right)\mathbbm{1}_{\{y^d \leq y^s\}} \\
&+ \left(py^s - vy^s - B(y^d - y^s)\right)\mathbbm{1}_{\{y^d > y^s\}}.
\end{split}
\end{align}

Given that the quantity $y^s$ is chosen by the buyer prior to observing $y^d$, the buyer maximizes 

\begin{align}
\begin{split}
    \mathbb{E}[\Pi(y^s)] =& \int_{0}^{y^s} \left(py^d - vy^s + g(y^s - y^d)\right)f_Y(y^d)dy^d  \\
    & +\int_{y^s}^{\infty} \left(py^s - vy^s - B(y^d - y^s)\right)f_Y(y^d)dy^d.
\end{split}
\end{align}

Given $\mathbb{E}_Y = \mu$, the previous equation reduces to 
\begin{align}
\begin{split}
    \mathbb{E}[\Pi(y^s)] =& (p-g)\mu - (v-g)y^s \\
&- (p-g+B) \int_{y^s}^{\infty}(y^d-y^s)f_Y(y^d)dy^d.    
\end{split}
\end{align}

Therefore, the first order conditions needed for an optimal allocation satisfy

\begin{align}
\frac{d}{dy^s}\mathbb{E}[\Pi(y^s)] &= (p - v + B) - (p - g + B)F_Y(y^s), \\
\frac{d^2}{d{y^s}^2}\mathbb{E}[\Pi(y^s)] &= -(p - g + B)f_Y(y^s) < 0.
\end{align}

Thus, $\mathbb{E}[\Pi(y^s)]$ reaches its global maximum at ${y^s}^* = F_Y^{-1}\left(\frac{p-v+B}{p-g+B}\right)$ (Theorem \ref{th:qloss_opt}, Target quantile minimizes the quantile loss). The above shows the existence of a unique optimal quantile $y_q^*= {y^s}^*$ from the consumers' demand density function that maximizes the buyer's expected profits. Such quantile $q$ accumulates $\frac{p-v+B}{p-g+B}$ from $F_Y$, which represents the ratio between the costs the buyer faces when exceeding or not the consumers' demand.

Therefore, companies can search for the quantile that maximizes their profits. Although there are methods to approximate the value of the quantile based on unit costs of overestimation and underestimation \citep{quantile_assesments}, in this work, we will assume that this quantile is known. Thus the \emph{News Vendor Problem} justifies the existence of an optimal quantile that will serve as input for the time series forecast for that quantile.

\section{Deep Learning Models}

\subsection{Universal Approximation Theorem}

Probabilistic time series forecasting consists of finding a conditional probability on covariates for a definite horizon. In practice, finding this probability function is not feasible. For this purpose, it is necessary to have a sufficiently flexible family of function approximators with sufficient accuracy. In particular, neural networks have good properties that allow estimating this function.

Aside from excellent empirical results, neural networks happen to also have interesting theoretical guarantees that states the universality of its approximation. The Universal Approximation Theorem
described in the Appendix (Theorem \ref{the:nn_univ}) 
states that any Lipschitz function 
(Definition \ref{def:lips}) 
can be approximated by a 3-layer Neural Network 
(Definition \ref{def:nn}). Although we present this result, there are versions of the theorem that prove a general version. See for example \cite{kratsios2020} and \cite{park2020minimum}.

In particular for \emph{Optimal Inventory Management}, sufficiently flexible and powerful approximators will allow to have excellent forecasting accuracy. We complete this theoretical guarantees with the empirical validation from Section \ref{section:experiments}.

\subsection{Neural Forecasting}

Beyond the fully connected neural networks that were proven universal approximators, in recent years, the use of deep learning for time series forecasting has become widespread. Since the victory of the hybrid ESRNN model in the M4 competition \citep{smyl2019esrnn}, fundamental advances in the area have emerged. This section began by recapping the relevant work starting with Recurrent Neural Networks and advances made in probabilistic modeling.


\emph{Recurrent Neural Networks} use hidden states propagated in time (RNN; \cite{elman90rnn, werbos1990backprop_through_time}). The hidden states work as encoders of the previous observations, which allows the model to be able to memorize relevant information from the series; thanks to this feature, this model has been used in different contexts of \emph{Sequence-to-Sequence models} (Seq2Seq) such as applied to natural language processing \citep{graves2013seq2seq, hermans2013DRNN} and machine translation \citep{sutskever2014seq2seq_translation, bahdanau2016neural}. Its widespread use has allowed an important development in the variants that have emerged; such is the case of new training techniques and designs that have proven to be functional as \emph{Long Short Term Memory} (LSTM; \cite{gers1999lstm}) and the \emph{Gated Recurrent Units} (GRU; \cite{chung2014GRU}).

One of the major milestones and precursors of recent advances in sequential modeling has been developing techniques to learn deeper representations through longer memory. This advance has been made possible by the use of convolutions and skip-connections within the recurrent structures. Such is the case of \emph{Temporal Convolutional Networks} (TCN; \cite{zico2018tcnn}) as well as \emph{Dilated RNN} (DRNN; \cite{chang2017dilatedRNN}), and \emph{WaveNet} for audio generation and machine translation \citep{vandenoord2016wavenet, dauphin2016cnn_word_modeling, kalchbrenner2016cnn_neural_translation}.


As pointed out in \cite{benidis2020dl_timeseries_review2}, the research done in fields such as machine translation and audio generation, framed in the sequence-to-sequence models literature, has been transferred to time series modeling with great success. This achievement is due to the ability of these models to learn complex temporal dependencies. Within the time series literature, there are models such as the \emph{Multi-Quantile Convolutional Neural Network} (MQCNN; \cite{wen2017mqrcnn}), the \emph{Exponential Smoothing Recurrent Neural Network} (ESRNN; \cite{smyl2019esrnn}), or the \emph{Neural Basis Expansion Analysis} (NBEATS; \cite{oreshkin2019nbeats}) and its exogenous version (NBEATS-x; \cite{olivares2021nbeatsx}).

An area of particular interest in time series forecasting with machine learning techniques has been probabilistic output. In this field, the aim is to know the conditional distribution of the future values of the time series and not only its expected value \citep{benidis2020dl_timeseries_review2}. There are two approaches to the subject. The first one assumes a parametric form of the output distribution such as the \emph{Autoregressive Recurrent Networks} (DeepAR; \cite{flunkert2017deepAR}) or the \emph{Associative and Recurrent Mixture Density Networks} (ARMDN; \cite{mukherjee2018armdn}). On the other hand, there is literature studying nonparametric models; these models produce quantiles directly by minimizing the quantile loss \citep{koenker_bassett_1978}. 

Although academia has resisted the widespread adoption of neural networks for time series forecasting \citep{makriadakis2018concerns}, these models have become state-of-the-art in many areas. For example, in weather forecasting \citep{nascimento2019weather}, energy markets \citep{lago2018DNN, dudek2020hybrid} and large online retailers \citep{flunkert2017deepAR, wen2017mqrcnn}. This performance has made its use more and more constant by the research community \citep{langkvist2014dl_timeseries_review1, benidis2020dl_timeseries_review2}.







\section{Quantile Regression}

\subsection{Quantile Loss}

The newsvendor problem states that profit maximization is reached at an optimal quantile of the distribution of the demanded quantity. In business scenarios where this distribution is unknown, the quantile $q$ can be heuristically defined. The theory developed by \cite{koenker_bassett_1978} on quantile regression helps us to generate forecasts for such quantiles. Thus, following the notation by \cite{wen2017mqrcnn}, models can be built to predict the conditional quantile 

\begin{equation}
\label{eq:cond_quantile}
y^{(q)}_{T+h}|y_T,
\end{equation}

of the target distribution given by 
\begin{equation}
\label{eq:dist_object}
\mathbb{P}(y_{T+h} \leq y^{(q)}_{T+h}|y_T) = q,
\end{equation}
seeking to minimize the quantile loss, given by 

\begin{equation}
\label{eq:quantile_loss}
\begin{split}
    QL_{q}\left(y_{T+h}, \hat{y}_{T+h}^{(q)}\right) =& q\max\left(0, y_{T+h}-\hat{y}^{(q)}_{T+h}\right) + \\
    &(1-q)\max\left(0, \hat{y}^{(q)}_{T+h}-y_{T+h}\right),
\end{split}
\end{equation}

where $q$ is the quantile to be forecasted.

As shown in Theorem \ref{th:qloss_opt} (Target quantile minimizes the quantile loss), the \emph{quantile loss} reaches its minimum at exactly the quantile $q$. Moreover, if $q = 0.5$, the Mean Absolute Error is recovered. Figure \ref{fig:quantile_loss} visually recovers this metric; the slope of the plot reflects the desired imbalance by overestimating or underestimating. The \emph{quantile loss} over the entire forecast horizon $H$ is calculated arithmetically, averaging the quantile losses of the individual forecasts.

\begin{figure}[htpb]
\begin{center}
\includegraphics[scale=0.25]{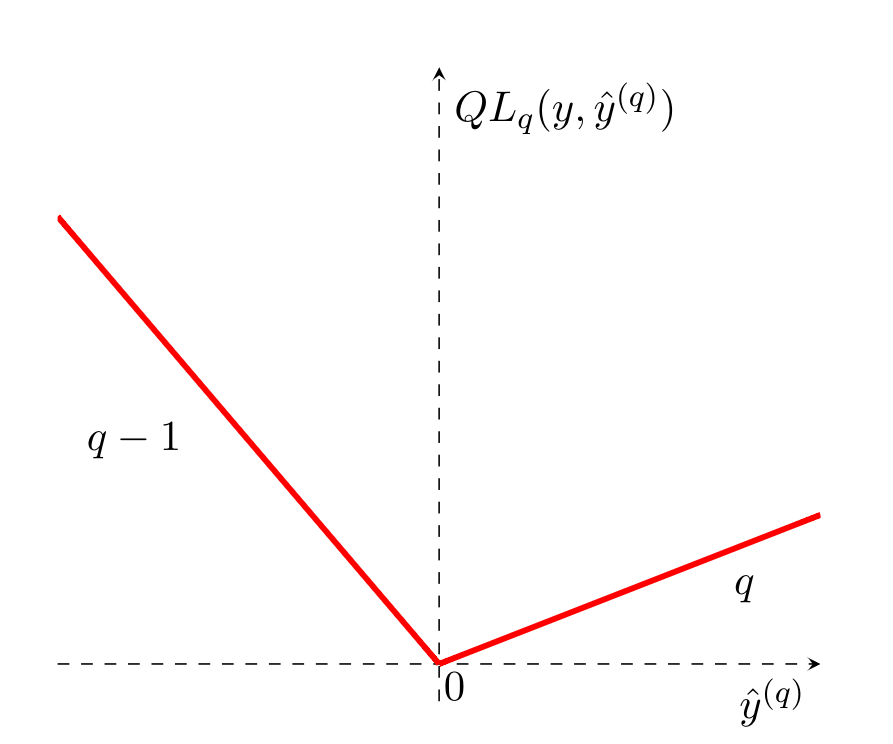}
\end{center}


\captionsetup{font=small}
    \caption{The Figure shows a plot of the \emph{quantile loss}. As can be seen, this loss allows to estimate a particular quantile of interest by selecting the parameter $q$. This quantile gives the slope of the loss.}
    \label{fig:quantile_loss}
\end{figure}

Finally, it is useful to calculate the sampling probability of the quantile forecasts as a measure of calibration. That is, the percentage of times that the actual value is below the quantile forecast. This metric is given by
\begin{equation}
\label{eq:calibration}
CL_{q}(\vy, \hat{\vy}^{(q)}) = \frac{1}{H} \sum^H_{h=1} 1_{\{y_{T+h} \leq \hat{y}_{T+h}^{(q)}\}}.
\end{equation}

The forecasts are expected to be well-calibrated, i.e., the difference between the percentage of times that the actual value is below the quantile forecast and the quantile is small. Formally, the equation

\begin{equation}
\label{eq:calibration_diff}
|CL_q(\vy, \hat{\vy}^{(q)}) - q| < \varepsilon,
\end{equation}

is verified for some $\varepsilon > 0$ and small. However, as observed by \cite{quantile_assesments} strategies can be defined where perfect calibration is approximated asymptotically without information of the process $y$.

\subsection{Quantile Autoregression}

The idea of robust regressions is ancient, as pointed out in \cite{Boscovich_Simpson1760quantile_regression} potentially as early as 1760. Subsequently, the framework of quantile regression was developed primarily in \cite{koenker_bassett_1978}; in their work, they extend the use of least squares to linearly estimate the conditional mean of $Y$ given $X$ for the quantile case.
Instead of minimizing the mean of the quadratic errors, the authors develop the idea of minimizing the mean of the quantile loss.
\cite{tata_sushanini_2012} reviews statistical characteristics of the model.



In time series, \emph{quantile autoregression with covariates} (\emph{QAR-X}) seeks to model the quantiles through a linear regression where the exogenous variables are specific lags of the time series. The QAR-X is a local model; it is defined by

\begin{equation}
\label{eq:qarx}
    y_{t}^{(q)} = \beta_0 + \sum_{l \in L} \beta_l y_{t -l} + \mathbb{\alpha} \vx,
\end{equation}

for each quantile $q$ and for each time series, where $L$ is a set of lags, and $\vx$ is a vector of exogenous variables, e.g., trend or day of the week. The forecast can be generated by a \emph{Seasonal Naive} (\ref{eq:seasonal_naive}) of the fitted values or in a recursive manner\footnote{Using one-step ahead forecasts iteratively to forecast a horizon greater than one is known as \textit{teacher forcing}.}. It consists of forecasting the next value and using this recursively to generate the next one. This method is characterized by accumulating errors \citep{benidis2020dl_timeseries_review2}. Finally, minimizing the quantile loss finds the best quantile regression model.

This model has been widely used for problems in different areas. In finance, it has been used mainly for portfolio construction to minimize the risk, encoded in a selected quantile $q$ \citep{xiao_koenker_2009,allen_powell_2011,lingjie_pohlman_2008}. Similarly, it has been applied in supply chain logistics for demand forecasting \citep{Bruzda2020} and in electricity-load to avoid blackouts \citep{elamin_2018}.
 \label{section:literature}
\section{Dilated Recurrent Neural Network}

\emph{Recurrent Neural Networks} (RNN) were first proposed and popularized in the work of \cite{rumelhart1986, elman90rnn,WERBOS1988339}. The fundamental idea of this architecture consists of connecting the hidden units of the neural network with themselves in a previous period. In each time step $t+1$, the network receives two inputs: some external value corresponding to that time step $t+1$ and the output of the hidden units in the previous time step $t$ \citep{elman90rnn}. 
Figure \ref{fig:rnn} describes the architecture of this model. All timesteps share the network weights. 

\begin{figure}[!htpb]
\centering     
\includegraphics[width=\textwidth]{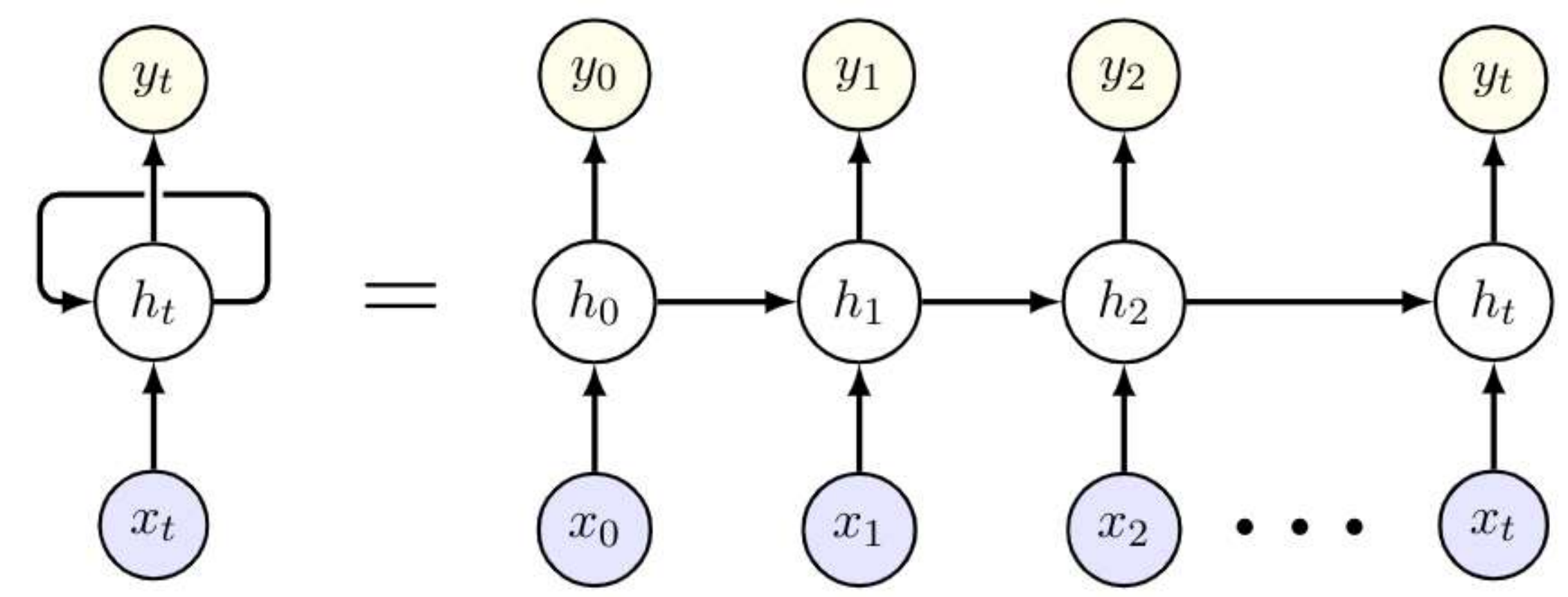}
\captionsetup{font=small}
\caption{The Figure shows the architecture of a Recurrent Neural Network (RNN). The network receives an external input $x_t$ and the output of the previous hidden unit $h_{t-1}$ to produce the output $y_t$. }
\label{fig:rnn}
\end{figure}



Although RNNs have shown great success in different applications, one of their main problems is handling very long sequences \citep{chang2017dilatedRNN}. In particular, RNNs present problems handling very long term dependencies in conjunction with short and medium-term dependencies; also, the training of this architecture presents vanishing and exploiting gradient and back-propagation through-time problems; finally, this back-propagation algorithm increases the training time \citep{pascanu2013difficulty}.

We use the \emph{Dilated Recurrent Neural Network} (DRNN; \cite{chang2017dilatedRNN}) model in our method, shown in Figure \ref{fig:dnn}. In this architecture the hidden unit at $t+1$ no longer depends on the hidden unit at $t$ but on the hidden unit at $t+1-l$ where $l > 1$ is the dilation size. This dilation size can be different for each of the hidden layers\footnote{Observe that if $l=1$ for each of the hidden layers, the RNN is recovered.}. This architecture alleviates gradient problems, improves computational time, and better learns long-term dependencies.

\begin{figure}[!htpb]
\centering     
\includegraphics[width=\textwidth]{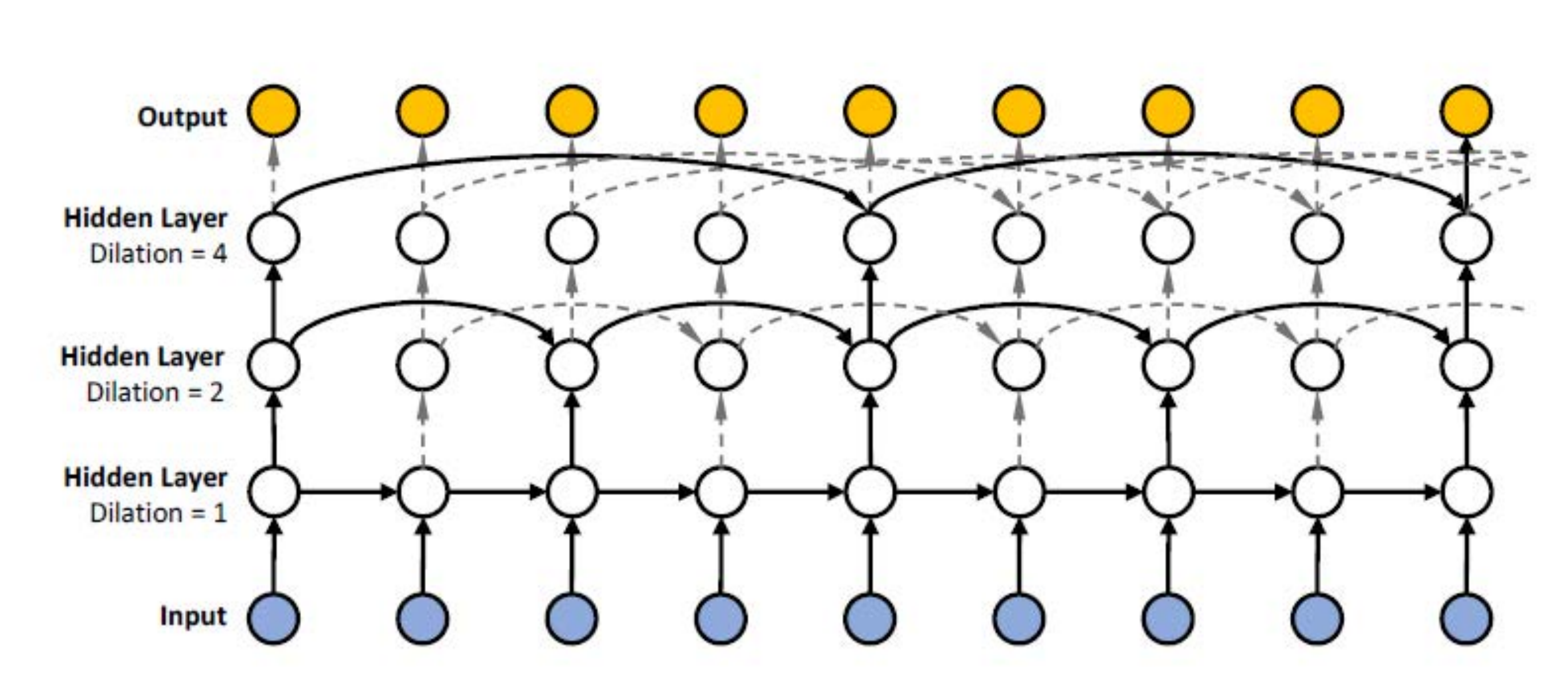}
\captionsetup{font=small}
\caption{The Figure shows the architecture of the \emph{Dilated Recurrent Neural Network} (DRNN). For each hidden layer, the size of the dilation changes. The hidden units of each layer at $t+1$ depend on the hidden units of the same layer at $t+1-l$ where $l$ is the dilation size. 
}
\label{fig:dnn}
\end{figure}

\section{Temporal Scaling}

One of the main inspirations for the model we propose is the \emph{Exponential Smoothing-Recurrent Neural Networks} model (ESRNN; \cite{smyl2019esrnn}). This model is hybrid: it scales the time series locally (for each time series) using \emph{Exponential Smoothing} (ES; \cite{HYNDMAN2002439}) (\ref{eq:exp_smooth}), decomposing it into levels and seasonality, to model the shared trends of all series using the RNN. This model turned out to be the winner of the fourth version of the Makridakis competitions (M4; \cite{makridakis2018m4competition_results}).

Scaling the time series before using the RNN is a crucial part of the ESRNN algorithm. In our proposal, we take up this idea. We process the time series locally (for each time series itself) in our model using the median. We then model the median residuals globally (the whole residuals of the time series together) with a DRNN. Finally, the model constructs the forecasts by adding the predicted residuals to the median signal of the time series.

Formally, the time series is given by $y_t$. We get the level of each time series trough its median and it is denoted by $l_t$ (\ref{eq:mqrnn1}). The difference between the level and the series is given by $z_t$ (\ref{eq:mqrnn2}). We use a DRNN to model the residuals $z_t$ jointly. The DRNN receives as arguments the residuals $z_t$, temporal exogenous variables denoted by $x_t$, and static exogenous variables denote by $s_t$ (\ref{eq:mqrnn3}). The model predicts the level using the Naive method (\ref{eq:naive}), that is, with the last observed value of the same level (\ref{eq:mqrnn4}). Finally, the model computes the forecasts by adding the predicted level of the series with the outputs of the DRNN (\ref{eq:mqrnn5}).

\begin{align}
    \text{Level:}       & \quad l_{t} = \text{median}(y_{t}) \label{eq:mqrnn1} \\
    \text{Residual:}    & \quad z_t = y_t - l_t \label{eq:mqrnn2} \\
    \text{NN Forecast:} & \quad \hat{z}_t = \text{DRNN}(z_t, x_t, s_t) \label{eq:mqrnn3}\\
    \text{Level Forecast:}       & \quad \hat{l}_{t} = \text{Naive}(l_{t}) \label{eq:mqrnn4} \\
    \text{Forecast:}   & \quad \hat{y}_{T+h} = \hat{l}_{T+h} + \hat{z}_{T+h} \label{eq:mqrnn5} 
\end{align}

\section{Multi-Quantile Loss}

Machine learning models such as the DRNN we use in our proposed model can generate multiple outputs at the same time \cite{xu2019survey}. The above allows us to generate multiple quantiles for the same model. This ability is achieved through the use of \emph{multi-quantile loss} 

\begin{equation}
\label{eq:mqloss}
    MQL_Q\left(\vy, \left(\hat{\vy}^{(q)}\right)_{q \in Q}\right) = \frac{1}{H|Q|} \sum_{q \in Q} \sum^{H}_{h=1}QL_{q}\left(y_{T+h}, \hat{y}_{T+h}^{(q)}\right),
\end{equation}

where $H$ is the forecast horizon, and $Q$ is the set of quantiles \citep{wen2017mqrcnn}. To calculate the \emph{multi-quantile loss}, the quantile loss (Equation (\ref{eq:quantile_loss})) is arithmetically averaged for each forecast window and each time series for multiple horizons and multiple quantiles.

As can be observed, the quantile loss $QL_q$ is not comparable between time series because the error is calculated locally, i.e., for each of the time series; in particular, this problem arises with product demand data where each of them may exhibit different levels from each other. Finally, the described metric is calculated for a set of time series by simply calculating the arithmetic average of the loss of each time series.




 \label{section:method}
\section{El Globo Demand Dataset}

In order to evaluate our proposed model, we selected a sample of products from one of the largest companies in Mexico, El Globo. This sample consists of daily sales of 10,000 products selected for having the highest number of sales during 2019—the time series range from 2018-04-05 to 2020-02-26. As data preprocessing, missing values (days without sales) were replaced with zeros to complete the time series (for a total of 693 observations per product). 

Also, the proposed model is capable of receiving exogenous variables and static variables (for each time series). The day of the week was chosen as the exogenous variable, while the static variables correspond to the sales center and the product family. The sales center corresponds to the establishment where the purchase was made. The product family corresponds to the product category.

\section{Training Methodology}

\label{sec:training}

We divide the data into training, validation, and test sets, as shown in Figure \ref{fig:training-meth}, to train the proposed model on the dataset described above. The training set spans from April 5, 2018, to November 28, 2018; the validation set from November 29, 2018, to February 20, 2019; finally, the test set from February 21, 2019, to February 26, 2020 (corresponding to 53 weeks). The $.3$, $.5$, $.7$ and $.9$ quantiles (correspondingly the $30$, $50$, $70$ and $90$ percentiles) were chosen as they were the ones requested by the operation owning the dataset \footnote{The operations managers were interested in a symmetric interval around the median covering the 40\% of cases and an upper case, the 90 percentile.}.

The optimal hyperparameters of the model were chosen by observing those that minimized the \emph{multi-quantile loss} ($MQL$) in the validation set. The validation set forecasts were generated with a rolling window each week with a 7-day horizon, considering the task horizon. This technique is known as time series cross-validation \citep{hyndman_book}. To carry out the hyperparamter optimization, we used a Bayesian iterative algorithm (HyperOpt; \cite{bergstra2013hyperopt}) using \textit{python} over the validation set. Table \ref{table:hp_space} lists the hyperparameter space over which the optimization was performed while Table \ref{table:hpoptimal} lists the optimal set of hyperparameters.

\begin{figure}[!htpb]
\centering     
\includegraphics[width=\textwidth]{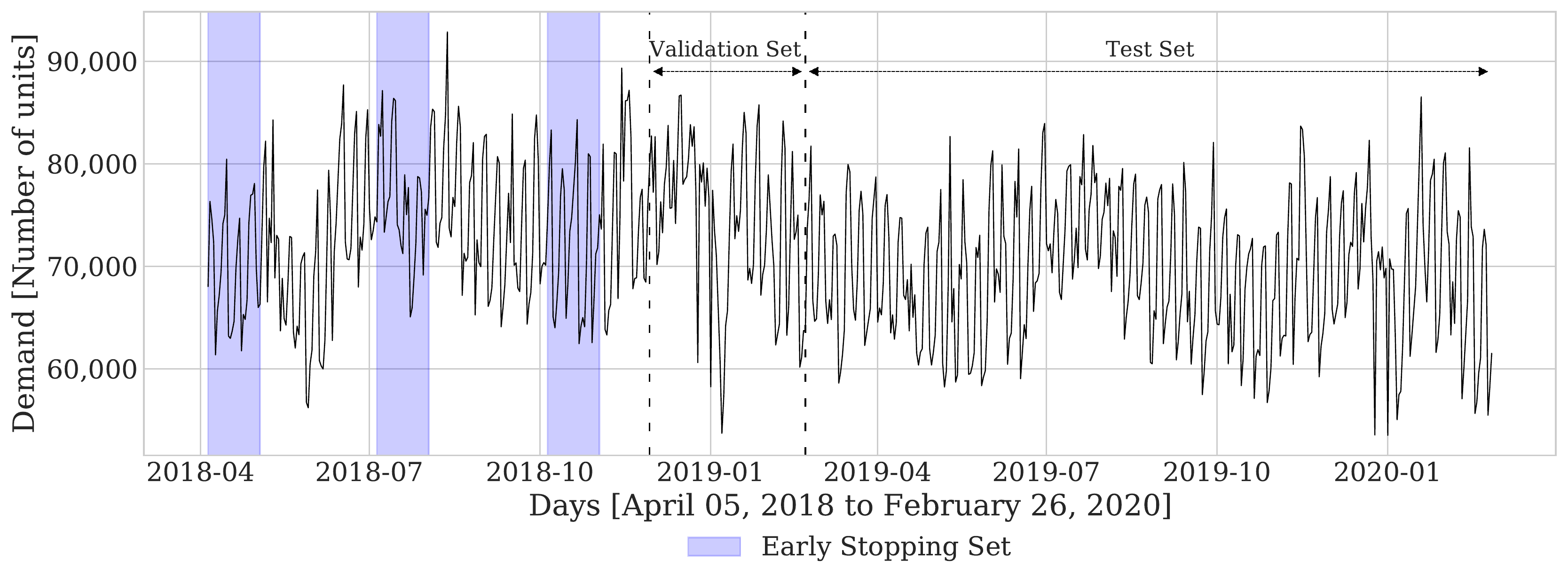}
\captionsetup{font=small}
\caption{Representation of train, validation and test sets. The validation loss provides the signal for hyperparameter optimization. Validation forecasts were calculated using 7-day-ahead rolling windows. Early stopping sets were selected at random (highlighted windows) during the test phase; it consists of three windows of four weeks each. Forecasts for test set were constructed using 7-day-ahead rolling windows.}
\label{fig:training-meth}
\end{figure}

The test set consists of forecasting an entire year, divided into 53 weeks. Forecasts were generated each week with a 7-day horizon to reflect the operational performance that many companies follow for forecast generation. For each week of the test set, the model was retrained with the previous data and the optimal hyperparameters found using the validation set. We used early stopping using three randomly selected windows of four weeks each as regularization. Figure \ref{fig:training-meth} recovers this procedure visually. Finally, the final forecasts were obtained by ensembling with the median ten instances of the trained model (parameters were randomly initialized in each instance controlled by a seed). 

\begin{table}[tbp]
\centering
\resizebox{\textwidth}{!}{\begin{tabular}{llc}
\hline
\multicolumn{1}{c}{\textbf{Hyperparameter}} & \multicolumn{1}{c}{\textbf{Description}} & \multicolumn{1}{c}{\textbf{Considered Values}}  \\
\hline
\multicolumn{3}{c}{\textbf{{Architecture Parameters}}} \\
\hline
input\_size\_multiplier & Input size is a multiple of the output horizon. & [4] \\
output\_size & Output size is the forecast horizon for day ahead forecasting. & [7] \\
add\_nl\_layer & Insert a tanh layer between the RNN stack and the output layer. & [True, False] \\
cell\_type & Type of RNN cell. & [GRU, RNN, LSTM, ResLSTM]\\
dilations & Each list represents one chunk of Dilated LSTMS. & [[[1, 2], [4, 8]], [[1, 2]]]\\
state\_hsize & Dimension of hidden state of the RNN. & range(10, 100)\\
\hline
\multicolumn{3}{c}{\textbf{{Optimization and Regularization parameters}}} \\
\hline
learning\_rate & Initial learning rate for regression problem. & range(5e-4,1e-3) \\
lr\_decay & The decay constant allows large initial lr to escape local minima. & [0.5]\\
lr\_scheduler\_step\_size & Number of times the learning rate is halved during train. & [2] \\
per\_series\_lr\_multip & Multiplier for per-series parameters. & [1]\\
rnn\_weight\_decay &  Parameter of L2/Tikhonov regularization of the RNN parameters. & [0]\\
batch\_size  & The number of samples for each gradient step. & [32]  \\
n\_iterations & Maximum number of iterations of gradient descent. & [1,000] \\
early\_stopping & Consecutive iterations without validation loss improvement before stop. & [10] \\
gradient\_clipping\_threshold & Max norm of gradient vector. & range(10, 100)\\
noise\_std & Standard deviation of white noise added to input during training. & [0.001]\\
hyperopt\_iters & Number of iterations of hyperparameter search. & [50] \\
\hline
\multicolumn{3}{c}{\textbf{{Data Parameters}}} \\\hline
idx\_to\_sample\_freq & Step size between each window. & [7]\\
val\_to\_sample\_freq & Step size between each window for validation. & [7]\\
window\_sampling\_limit & Number of time windows included in the full dataset. & [all available] \\
\hline
\end{tabular}}
\captionsetup{font=small}
\caption{Hyperparameter search space for the MQ-DRNN and MQ-DRNN-s models. We select the hyperparameter set with the best performance in the validation set using the HyperOpt library.}
\label{table:hp_space}
\end{table}

\newpage
\section{Model pipeline}




Hyperparameter optimization was performed using a worker with 4 CPUs, 1 GPU and 15 GB of RAM. Forecast generation was performed in parallel using \textit{dask} \citep{dask} in \textit{python} on a cluster of 20 workers with the above configuration raised with \textit{coiled}{\footnote{Coiled is a cloud service developed to automate and scale tasks. \hyperlink{https://coiled.io/}{https://coiled.io/}}}. To perform parallel processing, we generated 53 corresponding partitions and uploaded the data to Amazon Web Services. Parallel processing reduced the computational time by more than 20 times. The full pipeline is described in Figure (\ref{fig:model-pipe}).

\begin{figure}[htpb]
\centering     
\includegraphics[width=\textwidth]{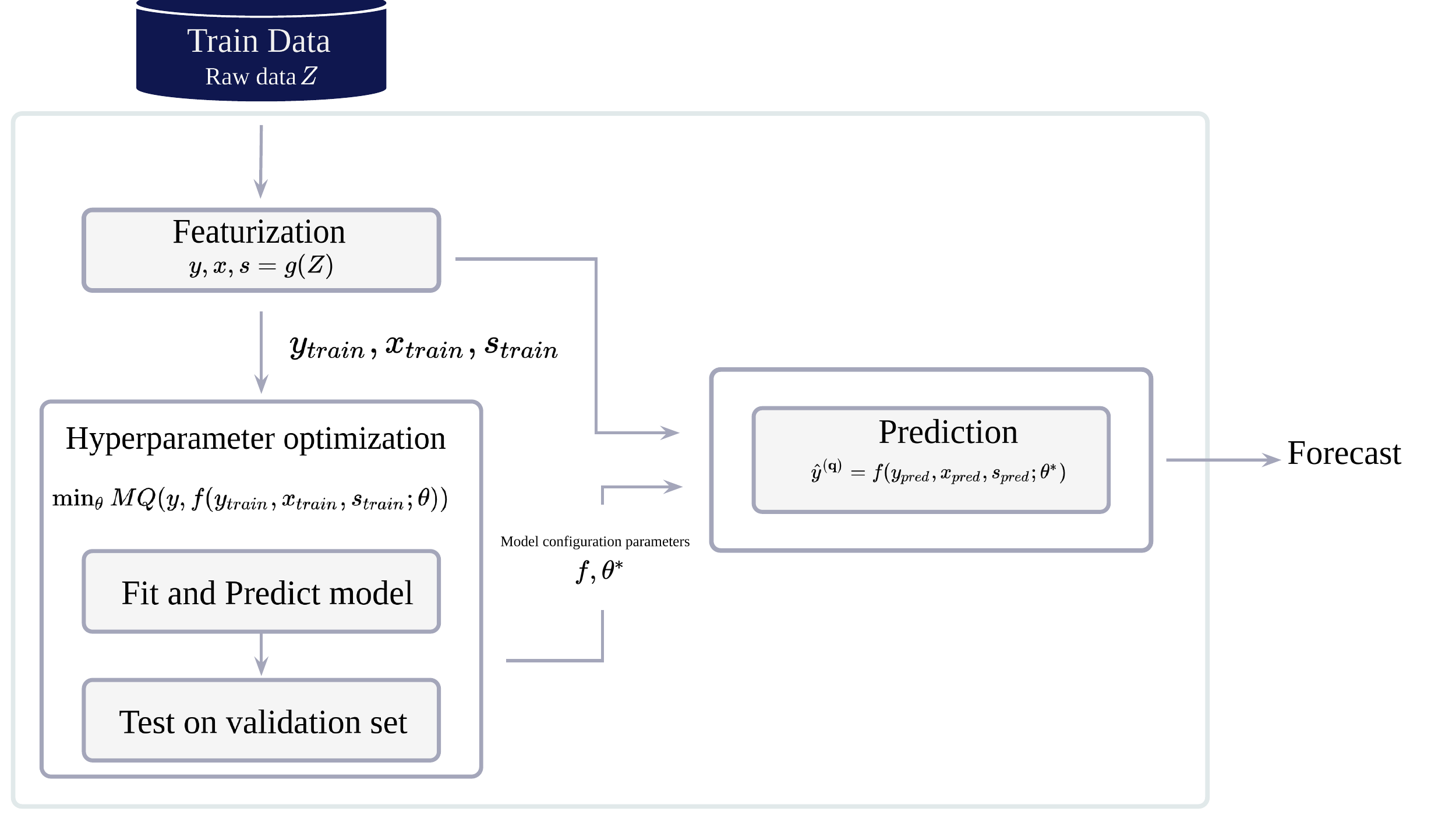}
\captionsetup{font=small}
\caption{The Figure describes the pipeline for obtaining the forecasts by the MQ-DRNN-s model. First, the set of 10,000 products $Z$ for which the forecast will be made is obtained. Then the objective variable $y$, the exogenous variable $x$ containing the day of the week, the static variables $s$ consisting of the sales center and the product family are obtained. Subsequently, hyperparameter optimization is performed on the validation set. With the optimal hyperparameters, the forecasts for each of the 53 weeks are generated in a rolling fashion.}
\label{fig:model-pipe}
\end{figure}

\newpage
\section{Main Results}

Table \ref{table:main_results2} shows the comparison between the \emph{QAR-X}, \emph{MQ-DRNN}, and \emph{MQ-DRNN-s} models on the test dataset. The first column shows the analyzed percentile. The QAR-X column shows the results associated with the autoregressive quantile model, while the MQ-DRNN column shows the results associated with the base model, and finally MQ-DRNN-s reports the results of our proposed model.

As can be seen, the proposed model is better for the evaluation metrics for the 30, 50, 70, and 90 percentiles (P30, P50, P70, P90). In the case of the 90 percentile, the models have comparable performance, with the QAR-X model being better by an insignificant difference.  The proposed model obtained a percentage improvement of 3.2\% compared to the QAR-X model and a percentage improvemente of 6\% compared to the MQ-DRNN model. Also, the models achieve adequate calibration for the percentiles analyzed as shown in Table \ref{table:main_results_cl}.

To test whether the results are statistically significant, we performed a 2- sided paired t-test following the practice in \citep{oreshkin2019nbeats}. The statistical test was performed for each QAR-X and MQ-DRNN model against our proposed model, MQ-DRNN-s. Table \ref{table:main_results_diff} shows the results. For each of the percentiles and each model, the difference is statistically significant at 99\%.



\begin{table}[!htpb]
\centering
\begin{tabular}{ccccc}
\multicolumn{4}{c}{Forecast Evaluation} \vspace{5mm}\\
\toprule
                        Percentile & QAR-X & MQ-DRNN & MQ-DRNN-s \\ \midrule
\multirow{1}{*}{P30}  &  1.15 & 1.17 & \textbf{1.11}  \\
\multirow{1}{*}{P50}  &  1.35 & 1.37 &\textbf{1.28}  \\ 
\multirow{1}{*}{P70}  &  1.21 & 1.24 & \textbf{1.16}  \\
\multirow{1}{*}{P90}  &  \textbf{0.64} & 0.70 & 0.66 \\ 
\bottomrule
\end{tabular}
\captionsetup{font=small}
\caption{Quantile forecast accuracy measures for 7-day-ahead demand during 53 weeks for 10,000 time series for models QAR-X, MQ-DRNN and MQ-DRNN-s. The reported metric is the \emph{quantile loss} (QL) at P30, P50, P70 and P90 percentile levels. Best results are highlighted in bold.}
\label{table:main_results2}
\end{table}

\begin{table}[!htpb]
\centering
\begin{tabular}{ccccc}
\multicolumn{4}{c}{Forecast Calibration} \vspace{5mm}\\
\toprule
                        Percentile & QAR-X & MQ-DRNN & MQ-DRNN-s \\ \midrule
\multirow{1}{*}{P30}   &  30.28 & 29.17 & 30.53 \\
\multirow{1}{*}{P50}   &  49.86 & 49.80 & 49.72 \\
\multirow{1}{*}{P70}   &  69.44 & 71.24 & 70.16 \\
\multirow{1}{*}{P90}   &  89.22 & 91.47 & 90.42 \\ 
\bottomrule
\end{tabular}
\captionsetup{font=small}
\caption{Calibration Metric (CL, \ref{eq:calibration}) for 7-day-ahead demand during 53 weeks for 10,000 time series for models QAR-X, MQ-DRNN and MQ-DRNN-s. The reported metric is the \emph{calibration} (CL) at P30, P50, P70 and P90 percentile levels. Qualitative analysis to verify that the forecasts are close to the predicted quantile.}
\label{table:main_results_cl}
\end{table}

\begin{table}[!htpb]
\centering
\begin{tabular}{ccc}
\multicolumn{3}{c}{Statistical performance comparison} \vspace{5mm}\\
\toprule
Percentile & QAR-X & MQ-DRNN \\ 
\midrule
\multirow{2}{*}{P30}  &  \textbf{-0.032} & \textbf{-0.051}   \\
                      & (0.00162) & (0.00161) \\
\multirow{2}{*}{P50}  &  \textbf{-0.069} & \textbf{-0.091} \\ 
                      & (0.00173) & (0.00182) \\
\multirow{2}{*}{P70}  &  \textbf{-0.045} & \textbf{-0.079} \\
                      & (0.00172) & (0.00193) \\
\multirow{2}{*}{P90}  &  \textbf{0.016} & \textbf{-0.047} \\
                      & (0.00152) & (0.00191) \\
\bottomrule
\end{tabular}
\captionsetup{font=small}
\caption{Average quantile loss difference against MQ-DRNN-s, computed as MQ-DRNN-s - QAR-X and MQ-DRNN-s - MQ-DRNN for P30, P50, P70 and P90 percentile levels. Standard error of the mean displayed in parenthesis. Bold entries are significant at the 99\% level (2-sided paired t-test).}
\label{table:main_results_diff}
\end{table}


Figures \ref{fig:forecasts_comparison_glb_closer}, \ref{fig:forecasts_comparison_glb_closer_2} and \ref{fig:forecasts_comparison_glb_closer_3} show the graphical results of the proposed model's forecasts for the whole year. It includes three time series in which we observe the difference in level for each of the forecasted percentiles and their ascending order. As an additional feature, we observe that the forecasts of all quantiles adapt to the behavior of the time series.

\newpage

\begin{figure}[!htpb]
\centering     
\includegraphics[width=\textwidth]{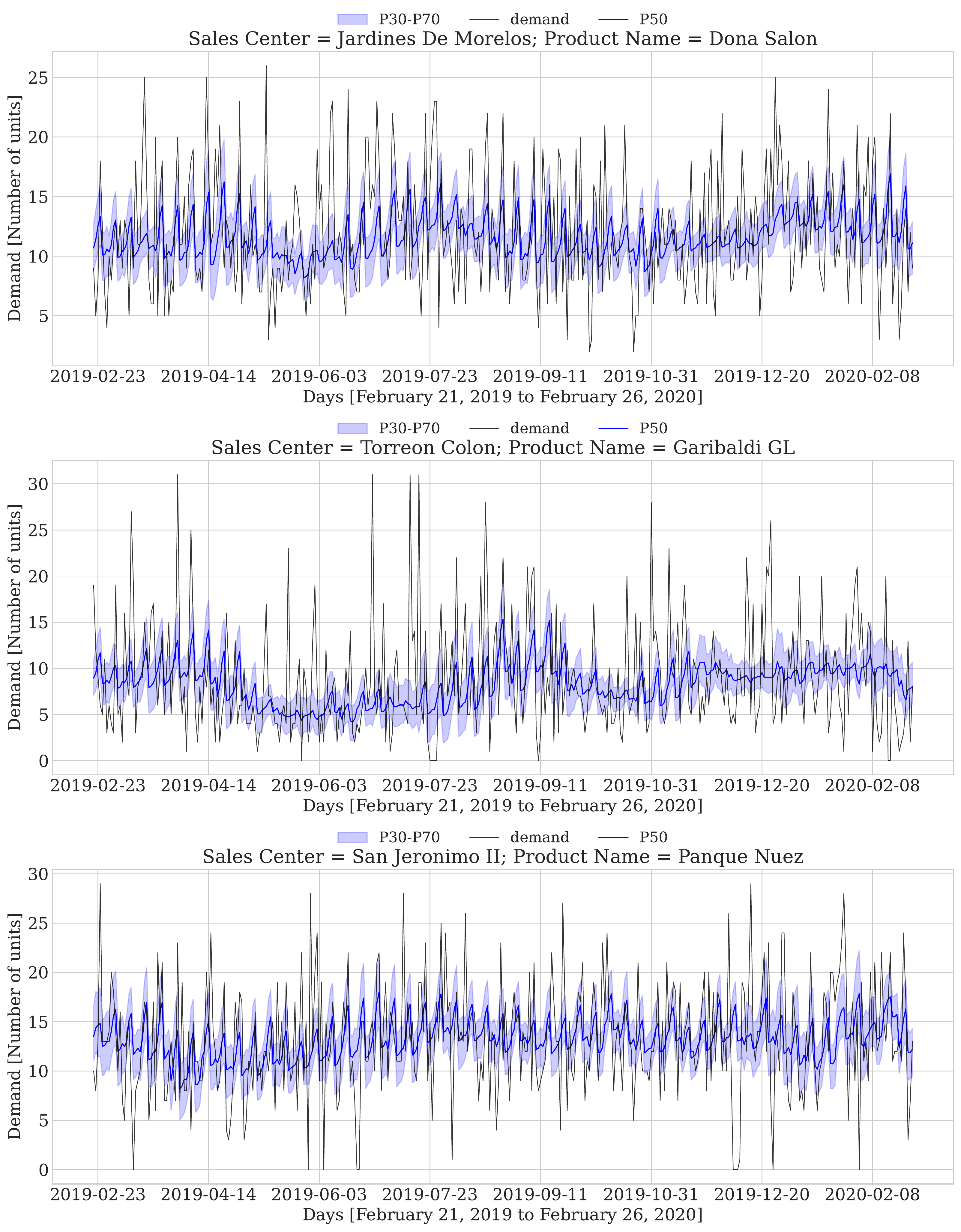}
\captionsetup{font=footnotesize}
\caption{The Figure shows the quantile forecasts of the MQ-DRNN-s model for percentiles P30, P50, P70, and true demand $y$ for 3 sample products from the dataset.  Finally the plot shows 53 sets of forecasts calculated each week.}
\label{fig:forecasts_comparison_glb_closer}
\end{figure}

\begin{figure}[!htpb]
\centering     
\includegraphics[width=\textwidth]{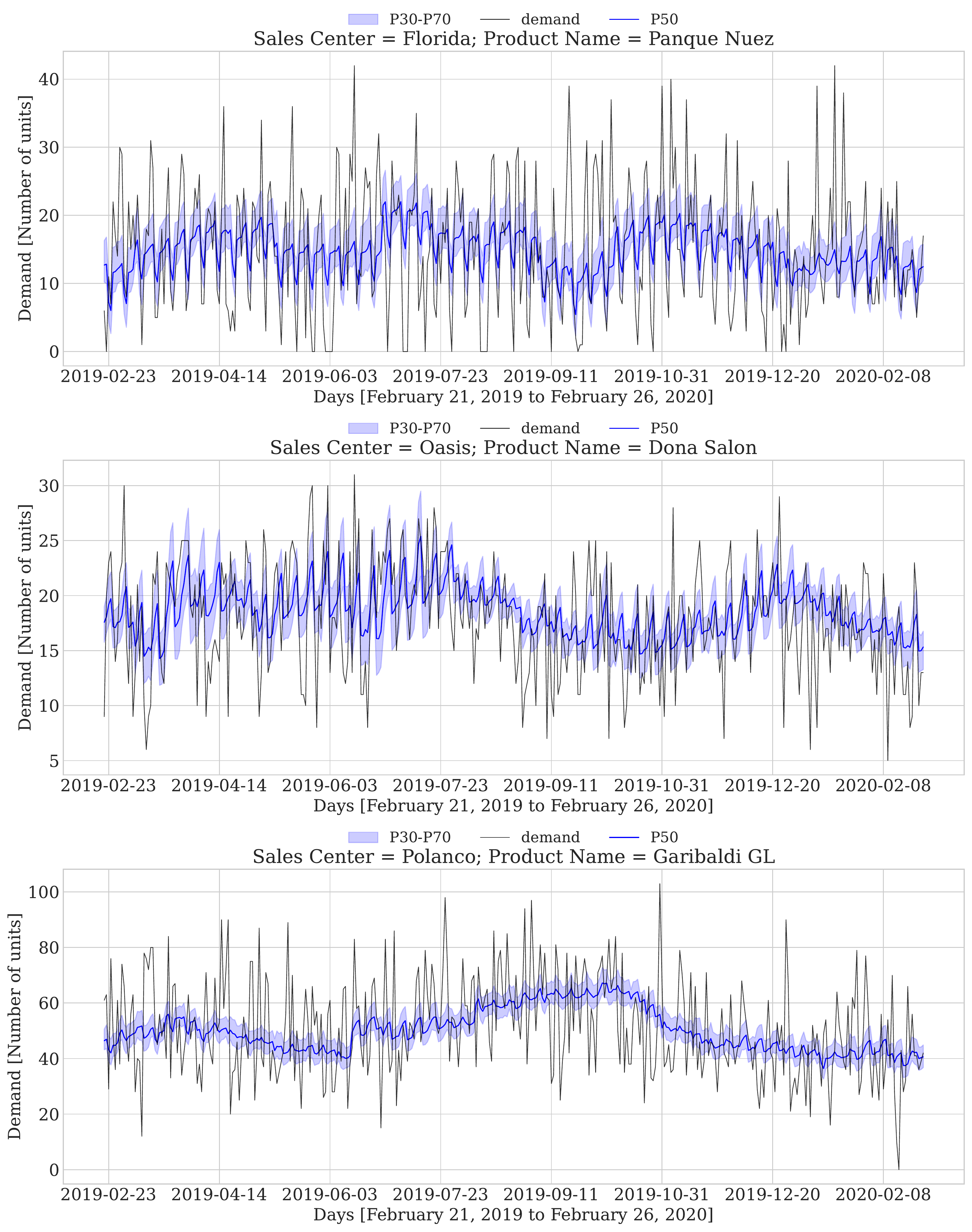}
\captionsetup{font=footnotesize}
\caption{The Figure shows the quantile forecasts of the MQ-DRNN-s model for percentiles P30, P50, P70, and true demand $y$ for 3 sample products from the dataset.  Finally the plot shows 53 sets of forecasts calculated each week.}
\label{fig:forecasts_comparison_glb_closer_2}
\end{figure}

\begin{figure}[!htpb]
\centering     
\includegraphics[width=\textwidth]{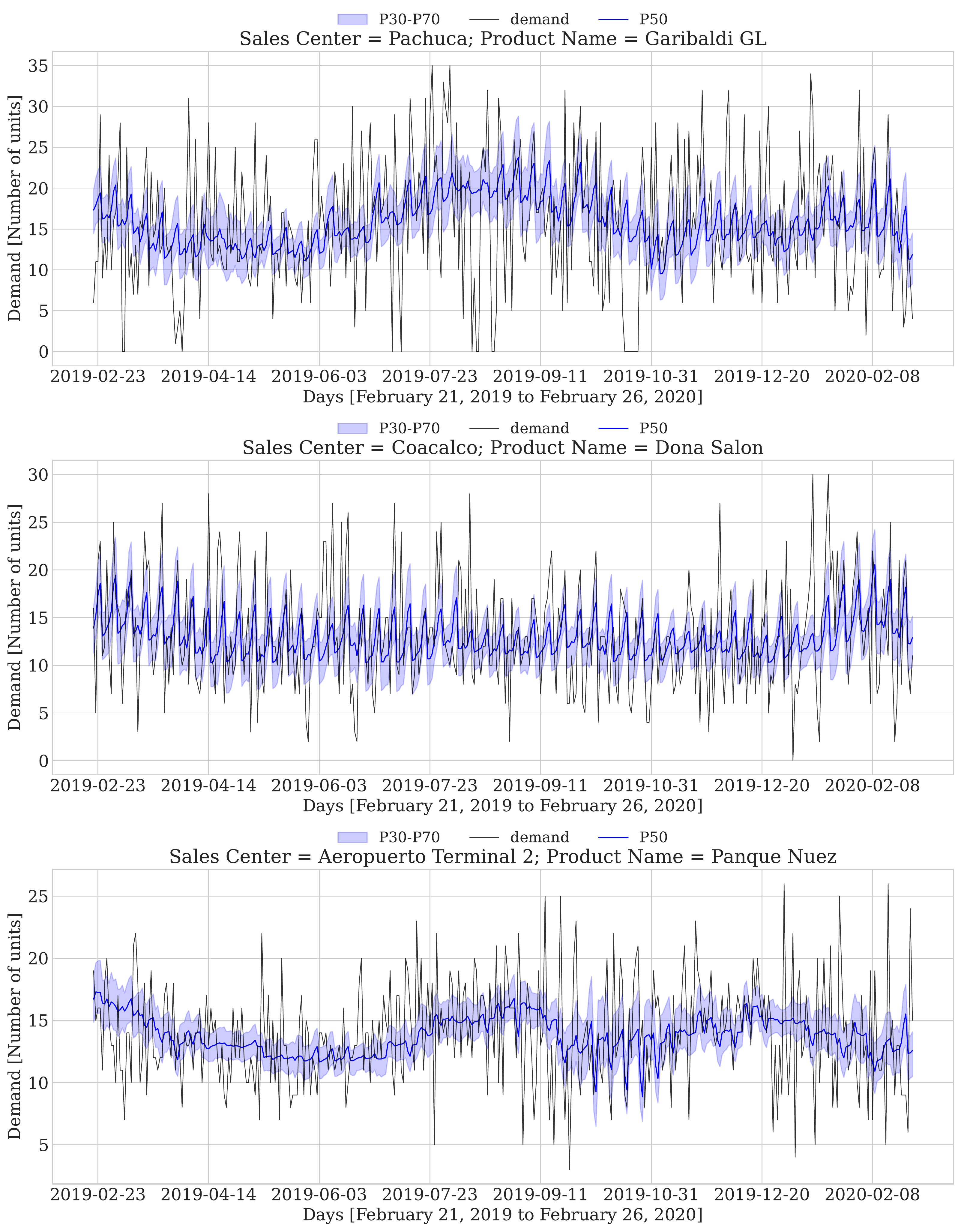}
\captionsetup{font=footnotesize}
\caption{The Figure shows the quantile forecasts of the MQ-DRNN-s model for percentiles P30, P50, P70, and true demand $y$ for 3 sample products from the dataset.  Finally the plot shows 53 sets of forecasts calculated each week.}
\label{fig:forecasts_comparison_glb_closer_3}
\end{figure}

 \label{section:experiments}
\section{Conclusion}

In this work we explore temporal scaling for multi-quantile time series forecasting using Recurrent Neural Networks (MQ-DRNN-s). This model showed a better performance of up to 3.1\% over an statistical benchmark (the quantile autoregressive model with exogenous variables, QAR-X) being this one better than the MQRNN without temporal scaling by 6\%. This results were found on a set of 10,000 time series of sales of one of the most important companies in Mexico over a 53-week horizon in a rolling fashion of 7-day ahead each week.

For companies seeking to make better decisions based on estimates of the future, the model proposed here offers three significant benefits: i) on the one hand, the data engineering required to put it into production is significantly reduced since an essential part of the data preprocessing occurs within the model, ii) the generation of multiple quantile scenarios is satisfied by training a single model; that is, the model does not have to be trained for each quantile, and finally, iii) the accuracy of the inferences generated over statistical baselines.

The results are consistent with the existing literature, demonstrating the remarkable capacity of machine learning models to forecast time series. These results also have significant business implications as this model's use would bring benefits in terms of accuracy and the generation of different scenarios.
 \label{section:conclusion}
\clearpage
\appendix





\section{Quantile Loss Optimality}

\begin{theorem}
\label{th:qloss_opt}
Let $q \in (0, 1)$ the target quantile, $y \in \R$ the target variable and  $\hat{y} \in \R$ the forecast. The quantile loss is defined by

$$
L_q(y, \hat{y}) = q \max(y - \hat{y}, 0) + (1 - q) \max(\hat{y} - y, 0).
$$
 
Then the minimum of the expected quantile loss, $
\hat{y}^* = \argmin_{\hat{y}} \mathop{\mathbb{E}_Y}[L_q(y, \hat{y})]
$ is the $q$ quantile, i.e.

$$
F_Y(\hat{y}^*) = q.
$$
\end{theorem}

\begin{proof}
We can rewrite the quantile loss as

$$
L_q(y, \hat{y}) = q(y - \hat{y}) \mathbbm{1}_{y \geq \hat{y}} + (1 - q)(\hat{y} - y)\mathbbm{1}_{y < \hat{y}}.
$$

Then we have

$$
\mathop{\mathbb{E}_Y}[L_q(y, \hat{y})] = q\int_{\hat{y}}^\infty(y - \hat{y})dF_{Y}(y) + (1 - q)\int_{-\infty}^{\hat{y}}(\hat{y} - y)dF_{Y}(y).
$$

Using the Leibniz integral  rule we have,

\begin{align*}
    \frac{\partial \mathop{\mathbb{E}_Y}[L_q(y, \hat{y})]}{\partial \hat{y}} &= q\left(\int_{\hat{y}}^\infty\frac{\partial(y - \hat{y})}{\partial \hat{y}}dF_{Y}(y) - (y - \hat{y})\biggr\rvert_{y = \hat{y}} \frac{\partial \hat{y}}{\partial \hat{y}}\right)\\ 
    &+ (1 - q)\left(\int_{-\infty}^{\hat{y}}\frac{\partial (\hat{y} - y)}{\partial \hat{y}}dF_{Y}(y)  + (\hat{y}-y)\biggr\rvert_{y = \hat{y}} \frac{\partial \hat{y}}{\partial \hat{y}}\right) \\ 
    &= -q \int_{\hat{y}}^\infty dF_Y(y) + (1 - q) \int_{-\infty}^{\hat{y}}dF_Y(y) \\
    &= -q(1 -  F_Y(\hat{y})) + (1 - q)F_Y(\hat{y}) \\
    &=-q + qF_Y(\hat{y}) + F_Y(\hat{y}) - qF_Y(\hat{y}) \\
    &=F_Y(\hat{y}) - q.
\end{align*}

Since the quantile loss is a convex function of $q$ the minimum is given by
$$
\frac{\partial \mathop{\mathbb{E}_Y}[L_q(y, \hat{y})]}{\partial \hat{y}} = 0,
$$

and therefore,
$$
F_Y(\hat{y}^*) = q.
$$

\end{proof}

\clearpage
\section{Optimal Hyperparameters}

\begin{table}[htpb]
\footnotesize
\centering
\begin{tabular}{llll}
\toprule            
Description & Hyperparameter & MQRNN &           MQRNN-s \\
\midrule
\multirow{6}{*}{Architecture} & input\_size\_multiplier &                 4 &                 4 \\
     & output\_size &                 7 &                 7 \\
     & add\_nl\_layer &             False &             False \\
     & cell\_type &           ResLSTM &              LSTM \\
     & dilations &  ((1, 2), (4, 8)) &  ((1, 2), (4, 8)) \\
     & state\_hsize &              65.0 &              96.0 \\
\cline{1-4}
\multirow{10}{*}{Optimization} & learning\_rate &          0.000621 &           0.00085 \\
     & lr\_decay &               0.5 &               0.5 \\
     & lr\_scheduler\_step\_size &                 2 &                 2 \\
     & per\_series\_lr\_multip &                 1 &                 1 \\
     & rnn\_weight\_decay &                 0 &                 0 \\
     & batch\_size &                32 &                32 \\
     & n\_iterations &              2000 &              2000 \\
     & early\_stopping &              True &              True \\
     & gradient\_clipping\_threshold &              46.0 &              73.0 \\
     & noise\_std &             0.001 &             0.001 \\
\cline{1-4}
\multirow{3}{*}{Data} & idx\_to\_sample\_freq &                 7 &                 7 \\
     & val\_idx\_to\_sample\_freq &                 7 &                 7 \\
     & window\_sampling\_limit &            100000 &            100000 \\
\bottomrule
\end{tabular}
\captionsetup{font=footnotesize}
\caption{The Table shows the optimal hyperparameters obtained with the signal from the validation set. The hyperparameters have been grouped under three main headings. The model architecture, which corresponds to the specific instance of the proposed model. The optimization hyperparameters, which is constituted by the characteristics associated with the optimizer. Finally, the data sampling parameters. The set of optimal hyperparameters for each model is shown.}
\label{table:hpoptimal}
\end{table}

\clearpage

\section{Basic Models}
\subsection{Naive}
\label{eq:naive}

The \emph{Naive} model is the simplest model to forecast time series. It simply assumes that the futures values will be the last one observed in the time series. Formally, the forecast for $y_{T+h}$ is given by

\begin{equation}
\label{eq:naive_eq}
\hat{y}_{T+h|T} = y_T \enspace \forall h \in \{1, ..., H\}.
\end{equation}

\subsection{Seasonal Naive}

\label{eq:seasonal_naive}

For time series seasonality is a repeated pattern of a fixed size. For example the same month (for yearly data), the same week of the year (for weekly data) or the same day (for daily data). The 
\emph{Sesonal Naive} model assumes that the next season value will be equal to the last values observed for that season. For example, with daily data considering a weekly seasonality the forecast for the next Monday will be the same as the previous Monday. Formally, using $m$ as the seasonal period the forecast for $y_{T+h}$ is given by

\begin{equation}
\label{eq:seasonal_naive_eq}
\hat{y}_{T+h|T} = y_{T + h - m} \enspace \forall h \in \{1, ..., H\}.    
\end{equation}

\subsection{Exponential Smoothing}

\label{eq:exp_smooth}

The Exponential Smoothing family of models seeks to describe the trend and seasonality of the data using weighted averages of past observations. In its additive version it is described as follows

\begin{align}
\label{eq:exp_smooth_eq}
\hat{y}_{t+h|t} &= l_t + \phi_h b_t + s_{t+h-m(k-1)}, \\ 
l_t &= \alpha (y_t - s_{t-m}) + (1 - \alpha)(l_{t-1} + \phi b_{t-1}), \\ 
b_t &= \beta(l_t - l_{t-1}) + (1 - \beta) \phi b_{t-1}, \\ 
s_t &= \gamma (y_t - l_{t-1} - \phi b_{t-1}) + (1 - \gamma)s_{t-m},
\end{align}

where $l_t$ denotes the series at time $t$, $b_t$ denotes the slope at time $t$, $s_t$ denotes the seasonal component of the series at time $t$ and $m$ denotes the season; $\alpha$, $\beta$, $\gamma$ and $\phi$ are smoothing parameters, $\phi_h = \sum_{i=1}^h \phi^i$ and $k$ is the integer part of $(h-1) / m$.  

The rest of the models in the family are constructed by changing the assumptions on trend, allowing it to be non-existent, additive or additive damped; and also on seasonality, allowing it to be non-existent, additive or multiplicative. For a more detailed explanation see \cite{hyndman_book}.

\clearpage

\section{Universal Approximation Theorem}

\begin{definition}
\label{def:lips}
Let $f: [0, 1]^d \rightarrow \mathbb{R}$ with $d \in \mathbb{N}$. We say $f$ is K-Lipschitz if there exists a constant $K > 0$ such that for each $\vx_1, \vx_2 \in [0, 1]^d$
\begin{equation*}
    |f(\vx_1) - f(\vx_2)| \leq K || \vx_1 - \vx_2 ||_2.
\end{equation*}
\end{definition}

\begin{definition}
\label{def:cell}
Let $d \in \mathbb{N}$. We say $C \subset \mathbb{R}^d$ is a cell if there exists $(l_i)_{i \in [d]}, (r_i)_{i \in [d]} \in \mathbb{R}^d$ such that
\begin{equation*}
    C = \times_{i = 1}^d [l_i, r_i].
\end{equation*}
\end{definition}

\begin{lemma}
\label{lem:sup}
Let $d \in \mathbb{N}$ and $f: [0, 1]^d \rightarrow \mathbb{R}$ a K-Lipschitz function. Let $\delta > 0$ and $P_\delta = \{C_1, \dots, C_N\}$ a partition of $[0, 1]^d$ where each $C_i$ is a cell with side length at most $\delta$. Let $\{\vx_1, \dots, \vx_N\}$ where $\vx_i \in C_i$ for each $i \in [N]$ then

\begin{enumerate}
    \item $\sup_{i \in [N]} \sup_{\vx' \in C_i} |f(\vx') - f(\vx_i)| \leq K\delta$.
    \item The function $g: \R^d \rightarrow \R$ given by $g(\vx) = \sum_{i=1}^N 1_{\{\vx \in C_i\}} f(\vx_i)$ satisfies
    $$
    \sup_{\vx \in [0, 1]^d} |f(\vx) - g(\vx)| \leq K\delta.
    $$
\end{enumerate}

\end{lemma}

\begin{proof}
$ $\newline
\textit{1.} Let $i \in [N]$ and $\vx' \in C_i$. Then as $\vx' \in C_i$ we have $||\vx'-\vx_i||_2 \leq \delta$. By Lipschitzness of $f$
\begin{align*}
    |f(\vx') - f(\vx_i)| &\leq K ||\vx'-\vx_i||_2 \\
                    &\leq K \delta.
\end{align*} 

\textit{2.} Let $\vx \in [0, 1]^d$. Since $P_\delta$ is a partition there exists $i \in [N]$ such that $\vx \in C_i$. Therefore $1_{\{\vx \in C_i\}} = 1$ and $1_{\{\vx \in C_j\}} = 0$ for each $i \neq j$. Thus $g(\vx) = f(\vx_i)$ and then by \textit{1}
$$
|f(\vx) - g(\vx)| = |f(\vx) - f(\vx_i)| \leq K\delta.
$$
\end{proof}

\begin{lemma}
\label{lem:cell}
Let $d \in \mathbb{N}$, $C = \times_{i=1}^d [l_i, r_i] \subset \mathbb{R}^d$ a cell and $\vx \in \mathbb{R}^d$. Then $\vx \in C$ if and only if for each $\gamma \in (0, 1)$ it follows that

$$
\sum_{i=1}^d 1_{\{x_i \geq l_i\}} + 1_{\{x_i \leq r_i\}} \geq 2 d - 1+ \gamma.
$$

\end{lemma}

\begin{proof}
\textit{Necessity.} Let $\vx \in C$ and $\gamma \in (0, 1)$. Then since $C$ is a cell $l_i \leq x_i \leq r_i$ for each $i \in [d]$. Therefore $1_{\{x_i \geq l_i\}} = 1$ and $1_{\{x_i \leq r_i\}} = 1$ for each $i \in [d]$. Thus
$$
\sum_{i=1}^d 1_{\{x_i \geq l_i\}} + 1_{\{x_i \leq r_i\}} = \sum_{i=1}^d 2 = 2 d.
$$

Since $\gamma \in (0, 1)$ it is true that $ \gamma - 1 < 0$ and therefore $ 2 d \geq 2d - 1 + \gamma$. Thus

$$
\sum_{i=1}^d 1_{\{x_i \geq l_i\}} + 1_{\{x_i \leq r_i\}} = 2 d \geq 2d - 1 + \gamma.
$$

\textit{Sufficiency.} Let $\vx \in \mathbb{R}^d$ and suppose that for each $\gamma \in (0, 1)$ it is true that

$$
\sum_{i=1}^d 1_{\{x_i \geq l_i\}} + 1_{\{x_i \leq r_i\}} \geq 2d - 1 + \gamma.
$$

Towards a contradiction suppose $\vx \notin C$. Then there exists $i^* \in [d]$  such that $x_{i^*} \leq l_i$ or  $x_{i^*} \geq r_i$; if $x_{i^*} \leq l_{i^*}$ then $1_{\{x_{i^*} \geq l_{i^*}\}} = 0$ and $1_{\{x_{i^*} \leq r_{i^*}\}} = 1$. On the other hand if $x_{i^*} \geq r_{i^*}$ then $1_{\{x_{i^*} \geq l_{i^*}\}} = 1$ and $1_{\{x_{i^*} \leq r_{i^*}\}} = 0$. Therefore $1_{\{x_{i^*} \geq l_{i^*}\}} + 1_{\{x_{i^*} \leq r_{i^*}\}} = 1$ and $1_{\{x_i \geq l_i\}} + 1_{\{x_i \leq r_i\}} \leq 2$ for each $i \neq i^*$. Thus 
\begin{align*}
\sum_{i=1}^d 1_{\{x_i \geq l_i\}} + 1_{\{x_i \leq r_i\}} &= 1_{\{x_{i^*} \geq l_{i^*}\}} + 1_{\{x_{i^*} \leq r_{i^*}\}} + \sum_{i \neq i^*} 1_{\{x_i \geq l_i\}} + 1_{\{x_i \leq r_i\}} \\
&\leq 1 + 2 (d -1) \\
& = 2d - 1.
\end{align*}

Which is a contradiction since $\gamma \in (0, 1)$.

\end{proof}

\begin{lemma}
\label{lem:g_tau}
Let $\tau \geq 0$ and $x \in \mathbb{R}$. Let $\sigma: \mathbb{R} \rightarrow \mathbb{R}$ given by $\sigma(x) = \max(0, x)$ and $g_\tau: \mathbb{R} \rightarrow \mathbb{R}$ given by $g_\tau(x) = |1_{\{x \geq 0\}} - (\sigma(\tau x) - \sigma(\tau x - 1))|$. Then

\begin{equation*}
g_\tau(x) =\begin{cases}
\leq 1 & x \in \left[0, \frac{1}{\tau}\right] \\
0 & oc.
\end{cases}
\end{equation*}

\end{lemma}

\begin{proof}
Let's use cases. 

\textit{Case 1.} Let $x \in \mathbb{R}$ such that $x < 0$. Then $1_{\{x \geq 0\}} = 0$, $\sigma(\tau x) = 0$ and $\sigma(\tau x - 1) = 0$. Therefore, $g_\tau(x) = 0$.

\textit{Case 2.} Let $x \in \mathbb{R}$ such that $x > \frac{1}{\tau}$. Then $1_{\{x \geq 0\}} = 1$, $\tau x > 1$ and thus $\sigma(\tau x) = \tau x$ and $\sigma(\tau x - 1) = \tau x -1$. Therefore
$$
g_\tau(x) = |1 - (\tau x - \tau x + 1)| = 0.
$$

\textit{Case 3.} Let $x \in \mathbb{R}$ such that $x \in \left[0, \frac{1}{\tau}\right]$. Then $1_{\{x \geq 0\}} = 1$, $\tau x \in [0, 1]$ and thus $\sigma(\tau x) = \tau x$ and $\sigma(\tau x - 1) = 0$. Therefore
$$
g_\tau(x) = |1 - \tau x| = 1 - \tau x.
$$
\end{proof} 

\begin{definition}
\label{def:nn}
Let $d \in \mathbb{N}$, $L, N \in \mathbb{N}$, $\theta^{(l)}_{j,i}, \beta^{(l)}_j \in \mathbb{R}$ for each $i,j \in [N]$, $l \in [L]$, and $\sigma: \R \rightarrow \R$ given by $\sigma(x) = \max(0, x)$. A feed-forward Neural Network with $L+1$ layers and $N$ units or neurons per layer is a function $f: \mathbb{R}^d \rightarrow \mathbb{R}$ given by
\begin{align*}
    f^{(1)}_j &= x_j,  \\
    f^{(l)}_j &= \sigma\left(\beta^{(l-1)}_{j} + \sum_{i=1}^N \theta^{(l)}_{j,i} f^{(l-1)}_i \right), \\
    f(\vx) &=  \sigma\left( \beta^{(L+1)} + \sum_{i=1}^N \theta^{(L+1)}_i f^{(L)}_i\right).
\end{align*}

For $l \in [L]$, $l > 1$.

\end{definition}

\begin{lemma}
\label{lem:nn_ind}
Let $d \in \mathbb{N}$, $C = \times_{i=1}^d [l_i, r_i] \subset \mathbb{R}^d$ a cell. Then there exists a 2-layer Neural Network $\hat{h}^{(2)}_\tau$ parametrized by $\tau >0$ such that

$$
\lim_{\tau \rightarrow \infty} \int_{[0, 1]^d} |\hat{h}^{(2)}(\vx) - 1_{\{\vx \in C\}}|d\vx = 0.
$$
\end{lemma}

\begin{proof}

Intuitively the proof follows the following steps. At first we define a function of $\tau$ (a $\sigma$ differential as defined in the Lemma \ref{lem:g_tau}) that approximates the indicator function when $\tau \rightarrow \infty$. Based on this function we define the neural network to approximate the indicator function of $C$. Formally, Let $\gamma \in (0, 1)$ and $h:\mathbb{R}^d \rightarrow \mathbb{R}$ given by 
$$
h_C(\vx) = 1_{\{\sum_{i=1}^d 1_{\{x_i \geq l_i\}} + 1_{\{x_i \leq r_i\}} \geq 2d - 1 + \gamma\}}.
$$
Then by Lemma \ref{lem:cell} we have that $1_{\{\vx \in C\}} = h_C(\vx)$ for all $\vx \in C$. Then for $\tau > 0$ we define $\hat{1}^\tau: \R \rightarrow \R$ given by
$$
\hat{1}^{\tau}_{\{x >= 0\}} = \sigma(\tau x) - \sigma(\tau x - 1).
$$

For $x \in \mathbb{R}$. Observe that for $x, y \in \mathbb{R}$ it follows that $\hat{1}^{\tau}_{\{x \geq y\}} = \hat{1}^{\tau}_{\{x - y \geq 0\}}$ and $\hat{1}_{\{x \leq y\}} = \hat{1}_{\{y - x \geq 0\}}$. Then we construct $\hat{h}^{(1)}_\tau: \mathbb{R}^d \rightarrow \mathbb{R}$ given by 
$$
\hat{h}^{(1)}_\tau(\vx) = 1_{\{\sum_{i=1}^d \hat{1}^{\tau}_{\{x_i \geq l_i\}} + \hat{1}^{\tau}_{\{x_i \leq r_i\}} \geq 2d - 1 + \gamma\}}.
$$
And we construct $\hat{h}^{(2)}_\tau: \mathbb{R}^d \rightarrow \mathbb{R}$ given by 
$$
\hat{h}^{(2)}_\tau(\vx) = \hat{1}^\tau_{\{\sum_{i=1}^d \hat{1}^{\tau}_{\{x_i \geq l_i\}} + \hat{1}^{\tau}_{\{x_i \leq r_i\}} \geq 2d - 1 + \gamma\}}.
$$

The function $\hat{h}^{(2)}_\tau = \hat{h}^{(2)}$ is our 2-layer Neural Network. Then by triangle inequality
\begin{align*}
    |\hat{h}^{(2)}(\vx) - 1_{\{\vx \in C\}}| &= |\hat{h}^{(2)}_\tau(\vx) - h_C(\vx)| \\
    & \leq |\hat{h}^{(2)}_\tau(\vx) - \hat{h}^{(1)}_\tau(\vx)| + |\hat{h}^{(1)}_\tau(\vx) - h_C(\vx)|.
\end{align*}
Let's define $L_i = \left[l_i, l_i + \frac{1}{\tau}\right]$ and $L = \times_{i=1}^d L_i$; also $R_i = \left[r_i - \frac{1}{\tau}, r_i \right]$ and $R = \times_{i=1}^d R_i$. Then we have that $|\hat{h}^{(2)}_\tau(\vx) - \hat{h}^{(1)}_\tau(\vx)| \neq 0$ if and only if $\vx \in [0,1]^d \cap (L \cup R)$. Also $|\hat{h}^{(2)}_\tau(\vx) - \hat{h}^{(1)}_\tau(\vx)| \leq 1$ and
\begin{align*}
\int_{[0, 1] \cap L_i} dx_i &\leq  \frac{1}{\tau}, \\
\int_{[0, 1] \cap R_i} dx_i &\leq  \frac{1}{\tau}.
\end{align*}

Therefore we have that,
\begin{align*}
\int_{[0,1]^d} |\hat{h}^{(2)}_\tau(\vx) - \hat{h}^{(1)}_\tau(\vx)| d\vx &= \int_{[0,1]^d \cap (L \cup R)} |\hat{h}^{(2)}_\tau(\vx) - \hat{h}^{(1)}_\tau(\vx)| d\vx \\
& \leq \int_{[0,1]^d \cap (L \cup R)} d\vx \\
& = \int_{[0,1]^d \cap L} d\vx + \int_{[0,1]^d \cap R} d\vx \\
& = \prod_{i=1}^d\int_{[0, 1] \cap L_i} dx_i + \prod_{i=1}^d\int_{[0, 1] \cap R_i} dx_i \\
& \leq \left(\frac{1}{\tau}\right)^d + \left(\frac{1}{\tau}\right)^d \\
&= 2 \left(\frac{1}{\tau}\right)^d.
\end{align*}

By the other hand, we have that $|\hat{h}^{(1)}_\tau(\vx) - h_C(\vx)| \neq 0$ if and only if
$$
\sum_{i=1}^d 1_{\{x_i \geq l_i\}} + 1_{\{x_i \leq r_i\}} \neq \sum_{i=1}^d \hat{1}^{\tau}_{\{x_i \geq l_i\}} + \hat{1}^{\tau}_{\{x_i \leq r_i\}}.
$$

Then we have that,
\begin{align*}
1_{\{\sum_{i=1}^d 1_{\{x_i \geq l_i\}} + 1_{\{x_i \leq r_i\}} \neq \sum_{i=1}^d \hat{1}^{\tau}_{\{x_i \geq l_i\}} + \hat{1}^{\tau}_{\{x_i \leq r_i\}}\}} &\leq \sum_{i=1}^d 1_{\{1_{\{x_i \geq l_i\}} \neq \hat{1}^\tau_{\{x_i \geq l_i\}}\}} \\
&+\sum_{i=1}^d 1_{\{1_{\{x_i \geq l_i\}} \neq \hat{1}^\tau_{\{x_i \geq l_i\}}\}} \\
& \leq d + d \\
&= 2d.
\end{align*}
Observe that $1_{\{x_i \geq l_i\}} \neq \hat{1}^\tau_{\{x_i \geq l_i\}}$ for all $x_i \in [0, 1] \cap L_i$ and $1_{\{x_i \leq r_i\}} \neq \hat{1}^\tau_{\{x_i \leq r_i\}}$ for all $x_i \in [0, 1] \cap R_i$. Therefore
\begin{align*}
\int_{[0,1]^d} |\hat{h}^{(1)}_\tau(\vx) - h_C(\vx)| d\vx & \leq \int_{[0,1]^d \cap (L \cup R)} (2d) d\vx \\
& = 2d \left(\int_{[0,1]^d \cap L} d\vx + \int_{[0,1]^d \cap R} d\vx\right) \\
& = 2d \left(\prod_{i=1}^d\int_{[0, 1] \cap L_i} dx_i + \prod_{i=1}^d\int_{[0, 1] \cap R_i} dx_i \right) \\
& \leq 2d \left( \left(\frac{1}{\tau}\right)^d + \left(\frac{1}{\tau}\right)^d \right) \\
&= 4d \left(\frac{1}{\tau}\right)^d.
\end{align*}

Finally, we have that
\begin{align*} 
    \int_{[0, 1]^d} |\hat{h}^{(2)}(\vx) - 1_{\{\vx \in C\}}|d\vx &\leq \int_{[0,1]^d} |\hat{h}^{(2)}_\tau(\vx) - \hat{h}^{(1)}_\tau(\vx)| d\vx \\
    &+ \int_{[0,1]^d} |\hat{h}^{(1)}_\tau(\vx) - h_C(\vx)| d\vx \\
    &\leq 2 \left(\frac{1}{\tau}\right)^d + 4d \left(\frac{1}{\tau}\right)^d \\
    &= (2 + 4d)\left(\frac{1}{\tau}\right)^d.
\end{align*}
Therefore
$$
\lim_{\tau \rightarrow \infty} \int_{[0, 1]^d} |\hat{h}^{(2)}(\vx) - 1_{\{\vx \in C\}}|d\vx = 0.
$$
\end{proof}

\begin{theorem}
\label{the:nn_univ}
Let $d \in \mathbb{N}$ and $f: [0, 1]^d \rightarrow \mathbb{R}$ a K-Lipschitz function. Then for all $\varepsilon > 0$ there exists a 3-layer Neural Network $\hat{f}$ with $O\left(d\right(\frac{L}{\varepsilon}\left)^d\right)$ neurons and ReLU activations such that 
$$
\int_{[0, 1]^d} |f(\vx) - \hat{f}(\vx)| d\vx \leq \varepsilon.
$$
\end{theorem}

\begin{proof}
Let $\varepsilon >0$, $\delta = \frac{\varepsilon}{K}$ and $P_\delta = \{C_1, \dots, C_N\}$ a partition of $[0, 1]^d$ where each $C_i$ is a cell with side length at most $\delta$. By Lemma \ref{lem:nn_ind} there exists a 2-layer Neural Network $\hat{h}^{(2)}$ such that 

$$
\int_{[0, 1]^d} |\hat{h}^{(2)}(\vx) - 1_{\{\vx \in C\}}|d\vx \rightarrow 0.
$$

Let's define $\hat{f}: \mathbb{R}^d \rightarrow \mathbb{R}$ given by
$$
\hat{f}(\vx) = \sum_{i=1}^N \hat{h}^{(2)}(\vx) f(\vx).
$$

This is our 3-layer Neural Network. Observe that
$$
\hat{f}(\vx) \rightarrow \sum_{i=1}^N 1_{\{\vx \in C\}} f(\vx).
$$

Since $\hat{h}^{(2)}$ can approximate arbitrarily well $1_{\{\vx \in C\}}$ we have by Lemma \ref{lem:sup}
$$
\sup_{\vx \in [0, 1]^d} |f(\vx) - \hat{f}(\vx)| \leq K \delta.
$$

Hence
\begin{align*}
\int_{[0, 1]^d} |f(\vx) - \hat{f}(\vx)| d\vx &\leq \int_{[0, 1]^d} K \delta d\vx \\
&= K \delta  \\
&= K \left(\frac{\varepsilon}{K}\right) \\
&= \varepsilon.
\end{align*}

\end{proof}

\clearpage

\bibliographystyle{agsm}
\bibliography{citations.bib}

\end{document}